\documentclass[11pt]{article}

\usepackage[utf8]{inputenc}
\usepackage[T1]{fontenc}
\usepackage{lmodern}
\usepackage[margin=1in]{geometry}
\usepackage{amsmath,amssymb,amsthm,mathtools}
\usepackage{booktabs,multirow,array}
\usepackage{graphicx}
\usepackage{subcaption}
\usepackage{xcolor}
\usepackage[numbers,sort&compress]{natbib}
\usepackage[colorlinks=true,linkcolor=blue!60!black,citecolor=blue!60!black,urlcolor=blue!60!black]{hyperref}
\usepackage[capitalise,noabbrev]{cleveref}
\usepackage{enumitem}
\usepackage{microtype}
\usepackage[section]{placeins}
\graphicspath{{figures/}}

\newtheorem{proposition}{Proposition}

\theoremstyle{definition}

\theoremstyle{remark}

\newcommand{\ngram}{\textsc{N-gram}}
\newcommand{\shared}{\textsc{Shared}}
\newcommand{\merged}{\textsc{Merged}}

\title{Representation Drift in Continual Learning:\\
LoRA versus N-gram Memory}

\author{
  Hao Wu \\
  Illinois Institute of Technology \\
  \texttt{hwu71@hawk.illinoistech.edu}
}

\date{}

\begin{document}
\maketitle

\begin{abstract}
A language model that keeps learning online must not slowly erode its own base capabilities. We ask which part of a modern LLM should absorb continual updates, by measuring how each choice changes the model's internal representations. On a 48-layer mixture-of-experts model with a native pretrained N-gram memory, we train a seven-task chain under three update mechanisms with identical data and budgets: updating only the N-gram memory table, continually training one shared LoRA adapter, and training a fresh LoRA per task that is merged into the backbone. We capture fixed-context representations at every layer and task boundary. Updating only the N-gram memory changes the base model's representations of general math and code inputs by 26\%/15\% (relative RMS), against 64\%/35\% for merged and 97\%/39\% for shared LoRA. This drift does not grow with the number of tasks: it reaches 25\% after the first task and stays there for all seven, while both LoRA variants keep accumulating. Its generic behaviour is essentially unchanged (no code problem lost, math NLL $-0.01$), and every change is confined to inspectable, reversible table rows. The cost appears elsewhere: N-gram memory acquires new tasks more weakly and its stored tasks interfere with one another, losing as much old-task NLL as shared LoRA. We show that this forgetting is governed by drift on the learned tasks' own inputs, measured from the checkpoint at which each was learned, not by generic drift (Spearman $\rho=0.85$; $\rho=0.71$ on Llama-3-8B), and we derive the exact relations behind it: with a frozen head, every pairwise logit margin changes by the projection of drift onto a head-difference direction. These results support a design principle for long-running online learning: write new knowledge only into the memory table and keep the backbone fixed; the remaining problem, interference between stored tasks, then lies inside the memory.
\end{abstract}

\section{Introduction}
\label{sec:intro}

A deployed language model that learns continually faces a constraint that one-off fine-tuning does not: updates never stop. Any perturbation of the shared backbone that accumulates with the number of updates will, over a long enough horizon, erode the general capabilities the model was deployed for. The practical question is therefore not only how well a model learns each new task, but \emph{where} in the model new knowledge should be written so that the base stays intact as updates pile up.

Modern LLMs offer more than one place to write. Besides low-rank adapters on the backbone weights (LoRA~\citep{hu2022lora}), recent production models carry large native memories, such as hashed N-gram embedding tables pretrained jointly with the backbone~\citep{qwen38flashnext}. These are read through deterministic token addresses and touched sparsely. We compare three write channels on the same model, data and budget:
\begin{itemize}[leftmargin=*,itemsep=2pt]
  \item \textbf{\ngram{} memory}: only the rows of the native N-gram table that the current task addresses are trained; every Transformer weight, the router, the memory read-out and the LM head are frozen.
  \item \textbf{\shared{} LoRA}: one LoRA adapter covering attention, linear attention and every routed and shared expert is trained continually across tasks.
  \item \textbf{\merged{} LoRA}: a fresh adapter of the same shape is trained per task and merged into the backbone before the next task.
\end{itemize}

We judge them by what they do to the model's representations. At every task boundary of a seven-task chain we capture representations for identical fixed contexts at all 48 layers, the LM-head input and every MLP/MoE block. We measure them against two references. Notation follows the key--value view of AlphaEdit~\citep{fang2025alphaedit}: a layer's input representation is a key, the keys of several inputs are stacked into a matrix $K$, and the layer's response is $Y=KW$. With a frozen output head, every change in logits comes from drift of the keys at the head. \emph{Big-$K$ drift} compares general math and code inputs with the pretrained base and measures erosion of the base. \emph{Small-$k_i$ drift} compares each learned task's inputs with the checkpoint at which that task had just been learned and measures what happens to acquired knowledge afterwards.

\paragraph{Findings.}
\begin{enumerate}[leftmargin=*,itemsep=2pt]
  \item \textbf{N-gram memory leaves the base nearly untouched, and the effect does not accumulate.} Its big-$K$ drift is 26\%/15\% on math/code versus 64--97\%/35--39\% for LoRA. It reaches this level after the first task and stays flat for all seven, while both LoRA variants keep growing. The ordering holds at every sampled layer and answer position. Generic NLL and generation are essentially unchanged (\cref{sec:res-bigk,sec:res-rounds,sec:res-depth}).
  \item \textbf{Its per-round changes cancel.} The realised squared drift is 15\% of the sum of per-round squared changes, against 40--56\% for LoRA. Successive table updates move the generic representation back and forth rather than away (\cref{sec:res-cancel}).
  \item \textbf{Its cost is interference among stored tasks, not damage to the base.} On learned-task inputs N-gram memory drifts as much as shared LoRA, and it acquires difficult tasks more weakly. Merged LoRA retains learned tasks best (\cref{sec:res-smallk,sec:res-accuracy}).
  \item \textbf{Forgetting follows small-$k_i$ drift, not big-$K$ drift}, on Qwen3.8 ($\rho=0.85$) and on Llama-3-8B ($\rho=0.71$) (\cref{sec:res-dissociation}).
\end{enumerate}

\paragraph{Theory.} \Cref{sec:framework} gives the relations that connect these measurements. They are a radial/directional decomposition of drift, a telescoping identity that exposes cross-round interference, a bound of retention loss by small-$k_i$ drift, and an exact margin identity: with a frozen head, a learned token is lost only when drift along a head-difference direction exceeds its acquisition margin. For N-gram memory, an input whose addressed rows were never rewritten is exactly unchanged. This is why generic drift stays small, and why interference concentrates on inputs that share rows.

\paragraph{Position.} For a model that must learn online indefinitely, the binding constraint is the base, and among the three channels only N-gram memory keeps the base fixed with a perturbation that does not grow with the number of tasks. We therefore argue that new knowledge should be written only into the memory table. The weaknesses we observe, weaker acquisition and inter-task interference in shared rows, are then problems of memory management, such as addressing, capacity and per-row update rules, rather than reasons to resume editing the backbone.

\section{Related Work}
\label{sec:related}

\paragraph{Catastrophic forgetting in LLM fine-tuning.}
Forgetting under sequential training is a long-standing problem~\citep{mccloskey1989catastrophic,kirkpatrick2017ewc}. For LLMs, continual instruction tuning and domain adaptation have been shown to degrade general capabilities~\citep{luo2023empirical}, and LoRA has been reported to ``learn less and forget less'' than full fine-tuning~\citep{biderman2024lora}. Our comparison is between two LoRA continuation strategies and a non-LoRA memory, under matched data exposure.

\paragraph{Knowledge editing and null spaces.}
ROME and MEMIT~\citep{meng2022rome,meng2023memit} treat FFN projections as key--value stores and write facts by solving constraints on $KW$. AlphaEdit~\citep{fang2025alphaedit} projects updates onto the null space of the keys to be preserved, $K_0$, so that $K_0\,\Delta W=0$. Our big-$K$/small-$k$ notation comes directly from this view. The difference is that a null-space constraint protects against weight change at fixed keys, whereas under continual training the keys themselves drift (\cref{eq:kw-split}); that second term is what we measure.

\paragraph{Orthogonality and subspace methods.}
A family of methods constrains successive updates to be orthogonal to previously used subspaces, including orthogonal gradient descent~\citep{farajtabar2020ogd}, O-LoRA~\citep{wang2023olora} and InfLoRA~\citep{liang2024inflora}. These methods implicitly use parameter- or gradient-space geometry as a proxy for forgetting. Our margin identity (\cref{prop:margin}) makes explicit the output-space quantity that such constraints are ultimately trying to protect.

\paragraph{Representation-level analyses of forgetting.}
\citet{ramasesh2021anatomy} localise forgetting to higher layers through representation similarity; \citet{davari2022probing} show that representation forgetting measured by linear probes can be much smaller than accuracy loss suggests. CKA~\citep{kornblith2019cka} is the standard similarity measure; we report it only as an auxiliary diagnostic because it is invariant to isotropic scaling and orthogonal transforms, both of which change logits under a frozen head.

\paragraph{Memory-augmented and sparse-update continual learning.}
Product-key and memory layers~\citep{lample2019pkm,berges2024memory} provide large, sparsely addressed parameter stores; \citet{lin2025sparsememory} argue that sparse memory fine-tuning forgets less than LoRA or full fine-tuning. Hashed N-gram memories pretrained jointly with the backbone now appear in production LLMs, including the model we study~\citep{qwen38flashnext}, whose report motivates deterministic addressing partly by cheap offloading. We study such a memory as a continual-update channel. Its frozen backbone does not keep learned-task representations fixed, and low generic drift does not imply better retention.

\paragraph{Margins and robustness.}
Logit margins are classical in generalisation and robustness analysis~\citep{bartlett2017spectrally}. We use the exact relation between margin change and drift projected onto head-difference vectors as a bookkeeping device for forgetting, closely related to how certified-robustness arguments bound prediction changes by input perturbation norms.

\paragraph{Merging and update placement.}
Merging task-specific weight increments into a shared model is studied as task arithmetic~\citep{ilharco2023task}; our merged-LoRA arm applies it sequentially, one task at a time. Forgetting has also been shown to depend on how an update is obtained, not only on its size~\citep{shenfeld2025rlrazor}. We vary where the update is placed (a sparse memory table, one continually trained adapter, or a sequence of merged adapters) under identical data, and measure where in the network the resulting drift appears.

\section{Representation Drift and How It Enters Forgetting}
\label{sec:framework}

Notation follows our experiment reports. $F_t$ is the checkpoint after task $t$ ($F_0$ is the pretrained base). For a fixed input context $x$ and fixed token position, $h_t(x)$ is the hidden state entering the frozen LM head $U$, so that
\begin{equation}
z_t(x)=Uh_t(x)+b .
\end{equation}
All three update mechanisms studied here leave $U$ and $b$ untouched, so every change in logits is a change in $h$.

\subsection{Keys, responses, and two sources of change}
\label{sec:keys}

Our $K$ and $k$ follow the key--value view of knowledge editing~\citep{meng2022rome,fang2025alphaedit}. A linear layer is read as a map from keys to values, and the layer's input representation is the key. Stacking the keys of several inputs as rows gives a matrix $K$, and the layer's response to them is
\begin{equation}
\label{eq:kw}
Y=KW .
\end{equation}
\emph{Big $K$} denotes the keys of existing knowledge that should be preserved, the $K_0$ of AlphaEdit. Complete world knowledge cannot be observed, so we use fixed generic math and code probes as its proxy. \emph{Small $k$} denotes the keys of newly written knowledge. Once task $i$ has been learned, its keys $k_i$ also become old keys to be preserved.

For the same inputs, the change of response between rounds $i$ and $t$ splits exactly as
\begin{equation}
\label{eq:kw-split}
Y_t-Y_i=\underbrace{K_i\,(W_t-W_i)}_{\text{weight change at fixed keys}}+\underbrace{(K_t-K_i)\,W_t}_{\text{key drift through new weights}} .
\end{equation}
AlphaEdit's null-space constraint $K_0\,\Delta W=0$ controls the first term. The second term comes from the keys themselves. Every layer's input is that layer's key, so any upstream parameter update reaches downstream layers as key drift, even if their weights never change.

The three update mechanisms occupy different places in \cref{eq:kw-split}. Both LoRA variants have both terms at every adapted layer. N-gram memory changes no weight outside the table, so the first term vanishes at every layer and only key drift remains, originating at the table injection in the second layer. The LM head is frozen in all three. At the head the key is $h$ and the response is the logits, so the first term is identically zero and
\begin{equation}
\label{eq:head-keys}
Z_t-Z_i=(K_t-K_i)\,U^\top :
\end{equation}
every output change comes from key drift. What this paper calls \emph{representation drift} is the key drift $K_t-K_i$ at each layer, and in particular at the head. We first define how to measure it and then show how it enters loss and predictions exactly.

\subsection{Two references: big $K$ and small $k_i$}

Task $i$ has a fixed probe set $S_i$. For $t>i$ define the drift of an old-task input and its pooled magnitude,
\begin{equation}
\label{eq:Dit}
e_i^t(x)=h_t(x)-h_i(x),\qquad
D_i^t=\left(\frac1{|S_i|}\sum_{x\in S_i}\|e_i^t(x)\|_2^2\right)^{1/2}.
\end{equation}
This is \emph{small-$k_i$ drift}: task-$i$ inputs, measured from the checkpoint at which task $i$ was just learned. \emph{Big-$K$ drift} replaces the reference by $F_0$ and the inputs by a fixed generic probe set (math, code) that no task trains on. Neither the change $F_0\to F_i$ on $S_i$ (that is learning) nor big-$K$ drift is a measurement of forgetting of task $i$. In tables we report the scale-free pooled ratio
\begin{equation}
\label{eq:relrms}
r_F=\frac{\|K_{\mathrm{after}}-K_{\mathrm{before}}\|_F}{\|K_{\mathrm{before}}\|_F}
=\left(\frac{\sum_x\|h_{\mathrm{after}}(x)-h_{\mathrm{before}}(x)\|^2}{\sum_x\|h_{\mathrm{before}}(x)\|^2}\right)^{1/2},
\end{equation}
where the rows of $K$ are the keys of these inputs at that site. $r_F$ is an energy-weighted average of per-query relative drifts, not their arithmetic mean, and a percentage of representation change, not of lost knowledge.

\paragraph{Length and direction.} For a fixed query with old representation $k$ and new representation $k'$, let $\rho=\|k'\|/\|k\|$ and $c=\cos(k,k')$. Then exactly
\begin{equation}
\label{eq:radial}
\frac{\|k'-k\|^2}{\|k\|^2}=(\rho-1)^2+2\rho(1-c).
\end{equation}
The first term is radial (length) change, the second directional change. Pooled decompositions weight by old energy; per-query percentages must not be averaged first and then substituted. Linear CKA~\citep{kornblith2019cka} is reported only as an auxiliary statistic: it is invariant to isotropic scaling and orthogonal rotation, both of which change what a fixed head reads.

\subsection{Accumulation across tasks}

Adjacent-round drift on the same input is $v_{i,j}(x)=h_j(x)-h_{j-1}(x)$. On a fixed input the telescoping identity is exact:
\begin{equation}
\label{eq:telescope}
e_i^t(x)=\sum_{j=i+1}^{t}v_{i,j}(x),\qquad
\|e_i^t\|^2=\sum_j\|v_{i,j}\|^2+2\sum_{a<b}\langle v_{i,a},v_{i,b}\rangle .
\end{equation}
The inner products are the inter-task interference terms: aligned successive updates make drift grow linearly in the number of tasks, mutually orthogonal ones only as its square root, and opposing ones cancel so that the endpoint drift is smaller than the sum of per-round drifts. For a probe set stacked as $K^{(j)}$ with $v_j=K^{(j)}-K^{(j-1)}$, normalise by the base energy,
\begin{equation}
\label{eq:R}
D_t=\frac{\|K^{(t)}-K^{(0)}\|_F^2}{\|K^{(0)}\|_F^2},\qquad
E_j=\frac{\|v_j\|_F^2}{\|K^{(0)}\|_F^2},\qquad
R_t=\frac{D_t}{\sum_{j=1}^tE_j}.
\end{equation}
$R_t<1$ iff the cross-round inner products sum to a negative number (net cancellation), $R_t>1$ iff they sum to a positive number; $R_t=1$ does not imply pairwise orthogonality. Archived adjacent drifts are normalised by $\|K^{(j-1)}\|_F$, so $E_j$ is recovered by first multiplying by the pooled norm ratio $q_{j-1}=\|K^{(j-1)}\|_F/\|K^{(0)}\|_F$ and then squaring. Moreover $D_j-D_{j-1}-E_j$ equals twice the inner product of round $j$'s change with the accumulated change, so dividing by $2\sqrt{D_{j-1}E_j}$ gives their overall cosine.

\subsection{From drift to loss and to argmax}

With a frozen head the logit change at the head input is exact,
\begin{equation}
\label{eq:dz}
\delta z_t(x)=z_t(x)-z_i(x)=Ue_i^t(x).
\end{equation}
For a teacher-forced answer token $y$, since $\nabla_z\ell=p-e_y$ has norm at most $\sqrt2$, the cross-entropy obeys the strict bound
\begin{equation}
\label{eq:ce-bound}
|\ell(z_t,y)-\ell(z_i,y)|\le\sqrt2\,\|U\|_2\,\|e_i^t\|_2 ,
\end{equation}
and averaging over the task set (Jensen) gives
\begin{equation}
\label{eq:dnll-bound}
|\Delta\mathrm{NLL}_i^t|\le\sqrt2\,\|U\|_2\,D_i^t .
\end{equation}
This bound connects small-$k_i$ drift directly to retention loss. Big-$K$ drift cannot replace $D_i^t$ in it unless one additionally assumes that old-task queries are distributed like the generic probes.

The bound ignores direction. The second-order expansion makes this explicit:
\begin{equation}
\label{eq:second-order}
\Delta\ell=g_i^\top e_i^t+\tfrac12(e_i^t)^\top U^\top\Lambda_iUe_i^t+O(\|e_i^t\|^3),
\qquad g_i=U^\top(p_i-e_y),\quad \Lambda_i=\operatorname{diag}(p_i)-p_ip_i^\top .
\end{equation}
Drift that falls in the approximate null space of $g_i$ leaves NLL unchanged to first order; small drift along a sensitive head direction can cause large forgetting.

For greedy correctness, let the acquisition token margin be $m_i(x)=z_{i,y}-\max_{c\ne y}z_{i,c}$. Because every pairwise logit gap changes by at most $\|u_y-u_c\|\,\|e\|\le2\|U\|_{2\to\infty}\|e\|$,
\begin{equation}
\label{eq:argmax-cert}
2\|U\|_{2\to\infty}\,\|e_i^t(x)\|_2<m_i(x)\ \Longrightarrow\ \text{the greedy prediction of this token is unchanged.}
\end{equation}
If every teacher-forced answer token satisfies \cref{eq:argmax-cert}, exact match of the whole answer is preserved. This is conservative and sufficient only; generation accuracy must be measured separately and cannot be inferred from NLL.

\subsection{Margin budget along a fixed decision boundary}

\Cref{eq:argmax-cert} uses only $\|e\|$. The exact statement keeps direction. For reference token $y$ and competitor $c\ne y$ with head rows $u_y,u_c$, define
\begin{equation}
a_{yc}=u_y-u_c,\qquad m_{i,c}=z_{i,y}-z_{i,c},\qquad
d_{i,c}=\frac{m_{i,c}}{\|a_{yc}\|_2},\qquad
H_{i,c}^t=-\frac{a_{yc}^\top e_i^t}{\|a_{yc}\|_2}.
\end{equation}

\begin{proposition}[Margin identity]
\label{prop:margin}
With a frozen affine head, for every fixed competitor $c$,
\begin{equation}
\label{eq:margin}
m_{t,c}=m_{i,c}+a_{yc}^\top e_i^t=m_{i,c}-\|a_{yc}\|_2H_{i,c}^t .
\end{equation}
Hence $H_{i,c}^t=d_{i,c}$ means $y$ and $c$ tie, and $H_{i,c}^t>d_{i,c}$ means $c$ strictly overtakes $y$.
\end{proposition}
\begin{proof}
$m_{t,c}=a_{yc}^\top h_t+(b_y-b_c)=a_{yc}^\top(h_i+e_i^t)+(b_y-b_c)=m_{i,c}+a_{yc}^\top e_i^t$.
\end{proof}

This involves no linearisation; $H/d$ and the margin change are the same quantity, not a prediction and its validation. When the reference token strictly leads the full vocabulary at $F_i$ (all $m_{i,c}>0$, outside numerical tolerance), $d_i=\min_{c\ne y,\,a_{yc}\ne0}d_{i,c}$ is a safe radius: $\|e_i^t\|<d_i$ preserves the argmax (Cauchy--Schwarz), which sharpens \cref{eq:argmax-cert}. Positions where $y$ already trails or ties at $F_i$ receive no safe radius and are counted separately. The geometrically nearest competitor
\begin{equation}
c_{\mathrm{geom}}=\arg\min_{c\ne y,\ \|a_{yc}\|_2>0}\frac{m_{i,c}}{\|a_{yc}\|_2}
\end{equation}
need not be the second-highest logit, and the full-vocabulary argmax at $F_t$ must be recomputed because a different token may overtake.

Along a fixed boundary the telescoping identity becomes exact per round:
\begin{equation}
\label{eq:H-telescope}
H_{i,c}^t=\sum_{j=i+1}^tH_{i,j,c},\qquad H_{i,j,c}=-n_{yc}^\top(h_j-h_{j-1}),\qquad n_{yc}=\frac{a_{yc}}{\|a_{yc}\|_2}.
\end{equation}
We summarise sign coherence by $H_{\mathrm{net}}=\sum_jH_{i,j,c}$, $H_{\mathrm{gross}}=\sum_j|H_{i,j,c}|$ and the cancellation coefficient $\kappa=1-|H_{\mathrm{net}}|/H_{\mathrm{gross}}$ (for $H_{\mathrm{gross}}>0$). A positive projection consumes margin and a negative one adds to it.

\paragraph{Sequences.} At answer positions $s$ where the reference token strictly leads at $F_i$, let $\rho_{i,s}^t=\|e_{i,s}^t\|_2/d_{i,s}$ and $\rho_{i,\mathrm{seq}}^t=\max_s\rho_{i,s}^t$. If the greedy output at $F_i$ equals the reference sequence with every step strictly leading, and decoding at $F_t$ uses the same prompt, head, greedy rule and termination, then $\rho_{i,s}^t<1$ at every position (including the end token) suffices to preserve the whole sequence. The true \emph{first flip} is where the greedy trajectory first departs; after it, model change and prefix change are mixed. Initial fragility and later erosion are separated by reporting the distribution of acquisition radii $d_{i,s}$ together with the normalised budget
\begin{equation}
F_i^{t,*}=\max_{s,c}\frac{H_{i,s,c}^t}{d_{i,s,c}} ,
\end{equation}
where $F_i^{t,*}\ge1$ means some boundary is reached under the reference prefix.

\subsection{Where drift comes from: layers and update structure}

\paragraph{Layer recursion.} For the same full token sequence, let $F_\ell,F'_\ell$ be block $\ell$ before and after an update and $E_\ell$ the representation change entering it. Exactly,
\begin{equation}
\label{eq:layer}
E_{\ell+1}=S_\ell+P_\ell,\qquad S_\ell=F'_\ell(K_\ell)-F_\ell(K_\ell),\qquad P_\ell=F'_\ell(K_\ell+E_\ell)-F'_\ell(K_\ell),
\end{equation}
with $\|E_{\ell+1}\|_F^2=\|S_\ell\|_F^2+\|P_\ell\|_F^2+2\langle S_\ell,P_\ell\rangle_F$. Locally $E_{\ell+1}\approx J_\ell E_\ell+S_\ell$: upstream drift propagates through the block Jacobian and the block's own parameter change adds to it. Layer curves can therefore amplify, attenuate or cancel, and the difference between adjacent layer curves is not that layer's causal contribution. The recursion separates where drift is \emph{created} ($S_\ell$, the source at layer $\ell$) from how it is \emph{carried} ($P_\ell$), and gives an accumulation bound.

\begin{proposition}[Cross-layer accumulation]
\label{prop:layers}
Let $L_j$ be a Lipschitz constant of the updated block $F'_j$ on the segment between $K_j$ and $K_j+E_j$, and let the input to the first block be unchanged ($E_1=0$). Then
\begin{equation}
\label{eq:layer-bound}
\|E_{L+1}\|_F\ \le\ \sum_{\ell=1}^{L}\Bigl(\prod_{j=\ell+1}^{L}L_j\Bigr)\,\|S_\ell\|_F .
\end{equation}
\end{proposition}
\begin{proof}
From \cref{eq:layer}, $\|E_{\ell+1}\|_F\le\|S_\ell\|_F+\|P_\ell\|_F\le\|S_\ell\|_F+L_\ell\|E_\ell\|_F$; iterate from $E_1=0$.
\end{proof}

Drift at the output is thus a sum over sources, each amplified by the blocks after it. The bound is conservative: directions and cancellation between sources matter, and depth alone does not determine forgetting. Its value here is structural, because the three update mechanisms have very different sets of sources.
\begin{itemize}[leftmargin=*,itemsep=2pt]
  \item \textbf{LoRA}: every adapted block changes, so $S_\ell\neq0$ at every depth and drift is created throughout the network.
  \item \textbf{\ngram{} memory}: every block except the table injection is frozen, so $S_\ell=0$ for all $\ell$ but one. The whole right-hand side reduces to a single term, the injection change $\Delta J(x)$ of \cref{eq:ngram} carried through the frozen suffix. In particular, if $\Delta J(x)=0$ for an input, its representation is unchanged at every depth. This is the memory-table analogue of the layer-wise induction used for null-space continual learning~\citep{wang2021adamnscl}.
\end{itemize}
How large that single source is depends on the input: small for generic text, which rarely reads the rows a task writes, and large for old-task inputs, which share rows with later tasks. \Cref{sec:res-smallk} shows both in the layer profile.

If round $j$ injects a direct perturbation $s_{\ell,j}^{(i)}$ at layer $\ell$ on old-task inputs, then to first order
\begin{equation}
\label{eq:jacobian-sum}
e_i^t\approx\sum_{j=i+1}^{t}\sum_\ell J_{L\leftarrow\ell+1}^{(t)}s_{\ell,j}^{(i)}+\text{higher-order terms}.
\end{equation}
Unlike \cref{eq:radial,eq:telescope,eq:margin}, this is an approximation, and any attribution built on it must first report its reconstruction residual $\|v-\hat v\|/\|v\|$.

For a block captured at its input and aggregate output, the endpoint change splits exactly as
\begin{equation}
\label{eq:mlp-split}
F_\ell^{(7)}(k_\ell^{(7)})-F_\ell^{(0)}(k_\ell^{(0)})
=\underbrace{F_\ell^{(0)}(k_\ell^{(7)})-F_\ell^{(0)}(k_\ell^{(0)})}_{\text{input change, old map}}
+\underbrace{F_\ell^{(7)}(k_\ell^{(7)})-F_\ell^{(0)}(k_\ell^{(7)})}_{\text{map change, new input}} .
\end{equation}
Our captures contain the left-hand side and the input change but not the cross-forward $F_\ell^{(0)}(k_\ell^{(7)})$, so the two terms are not separately identified.

\paragraph{Shared LoRA.} If the same factors $C,D$ keep training, the round-$j$ weight change is exactly
\begin{equation}
\label{eq:shared-dW}
\Delta W_{\ell,j}=\delta C_{\ell,j}D_{\ell,j-1}+C_{\ell,j-1}\delta D_{\ell,j}+\delta C_{\ell,j}\delta D_{\ell,j},
\end{equation}
decoder overwrite, re-encoding of old codes, and coupling. The first two terms are first order in the current change because the old factors are given, not small, and they enter $s_{\ell,j}^{(i)}=\Delta W_{\ell,j}K_\ell^{(i)}$ directly. The chain $\|\delta D\|\to\|\delta Dk_i\|\to\|C\delta Dk_i\|\to\|JC\delta Dk_i\|\to-n^\top JC\delta Dk_i$ separates code overlap and functional change from signed boundary harm; only the last is forgetting-relevant.

\paragraph{Fresh/merged LoRA.} The old factors are frozen inside the merged weights and each round adds $\Delta W_{\ell,j}=C_{\ell,j}D_{\ell,j}$. There is no direct overwrite of old factors, but the new adapter still responds to old inputs, and cross-layer Jacobians and higher-order products remain. The two arms also differ in expressible capacity: the cumulative update has rank at most $r$ per module for shared LoRA and at most $tr$ for merged LoRA. For fixed layer inputs $K_1,K_2$ and target increments $Z_1,Z_2$, any cumulative update $C$ of rank at most $r$ that fits the new task to $\|CK_2-Z_2\|_F^2\le\varepsilon^2$ must incur
\begin{equation}
\label{eq:capacity}
\|CK_1-Z_1\|_F^2\ \ge\ \Bigl[\sum_{j>r}\sigma_j([Z_1\ Z_2])^2-\varepsilon^2\Bigr]_+ ,
\end{equation}
by the Eckart--Young bound on $\operatorname{rank}(C[K_1\ K_2])\le r$. Increasing cumulative rank removes this obstruction but does not guarantee that optimisation finds a low-interference solution.

\paragraph{Rank does not set the disturbance; direction does.}
Write a weight change in the row convention of \cref{eq:kw} through its SVD, $D=\sum_j\sigma_j v_ju_j^\top$, with $v_j$ input-side and $u_j$ output-side unit vectors. Its effect on old keys $K$, the first term of \cref{eq:kw-split}, is exactly
\begin{equation}
\label{eq:dk-spectrum}
\|KD\|_F^2=\sum_j\sigma_j^2\,\|Kv_j\|_2^2 ,
\end{equation}
because $KD=\sum_j\sigma_j(Kv_j)u_j^\top$ and the $u_j$ are orthonormal. The disturbance is set by the size of each update direction and by how strongly the old keys project onto it, not by the number of directions. A low-rank update that points at old keys can be very harmful; a high-rank update orthogonal to them is harmless. It vanishes exactly when $Kv_j=0$ for every direction with $\sigma_j>0$, which is the null-space condition of AlphaEdit.

\paragraph{When task updates must overlap.}
Let the update written by task $j$ be $\operatorname{vec}(\Delta W_j)=\alpha_ju_j$ with $\|u_j\|_2=1$, for $m$ tasks, and let $q=\operatorname{rank}[u_1,\dots,u_m]$ be the dimension they jointly span. Then
\begin{equation}
\label{eq:welch}
\frac1{m(m-1)}\sum_{a\ne b}(u_a^\top u_b)^2\ \ge\ \max\Bigl(0,\ \frac{m-q}{q(m-1)}\Bigr).
\end{equation}
This is Welch's bound~\citep{welch1974bounds}. The Gram matrix $G=[u_a^\top u_b]$ has trace $m$ and rank at most $q$, so $\|G\|_F^2\ge m^2/q$, and the off-diagonal part is $\|G\|_F^2-m$. Once the number of tasks exceeds the dimension available to their updates, squared overlap between task directions is unavoidable. For $m\le q$ the bound is zero: orthogonal updates are geometrically possible, though optimisation need not find them. This is the structural difference between the two LoRA arms. \shared{} routes every task through one factor pair, so successive increments reuse the same row and column spaces (\cref{eq:shared-dW}); \merged{} gives every task fresh factors, so the jointly spanned dimension can keep growing. Overlap in parameter space is not yet forgetting. It matters only through \cref{eq:dk-spectrum}, by landing on directions that old keys use, and then through \cref{eq:margin}, by projecting onto the old task's head-difference directions.

\paragraph{N-gram memory.} The native table $T$ is read at a fixed early layer through addresses determined by hashed token bigrams and trigrams of the raw input; the backbone, gates and projections are frozen. Write the read pattern of context $x$ as a fixed sparse matrix $R(x)$. Upstream of the injection nothing changes, so the injected representation changes only through the table,
\begin{equation}
\label{eq:ngram}
\Delta J(x)=G\bigl(H(x),(T+\Delta T)R(x)\bigr)-G\bigl(H(x),TR(x)\bigr),
\end{equation}
where $G$ is the frozen read-out (key/value projections, hidden-dependent gate and causal convolution) and $H(x)$ the unchanged pre-injection state. If $\Delta T R(x)=0$ on the full causal prefix of an old input, its entire forward pass is unchanged. Any read of a row written by a later task---a genuinely shared N-gram or a hash collision---breaks this, and the change then propagates through the frozen suffix, whose inputs have moved even though its parameters have not. N-gram memory can therefore have very small big-$K$ drift (generic inputs rarely read rows the tasks write) and still large small-$k_i$ drift (old-task inputs share rows with new tasks).

\section{Experimental Setup}
\label{sec:setup}

\paragraph{Backbone.} We use Qwen3.8-Flash-Next~\citep{qwen38flashnext} (revision \texttt{de4b8e4}), a 48-layer hybrid-attention (gated linear attention and full attention) MoE language model with 125B backbone parameters (about 6B active), 512 routed experts with top-10 routing plus shared experts, and four hyper-connection residual streams that are contracted before the LM head. It carries a native, jointly pretrained N-gram memory of 51.2B parameters: a table of $320{,}001{,}536$ rows of width 160, read at the second layer. Each token reads one row from each of 8 bigram and 8 trigram hash heads (16 rows, 2{,}560 values). The addresses depend only on raw token IDs (with history reset after EOS), and the rows pass through key/value projections, a hidden-state-dependent gate and a causal convolution before being injected into the residual streams. The model runs in BF16. The final normalisation is the identity, so the captured head input determines logits exactly.

\paragraph{Task chain.} Seven tasks, learned in the fixed order
SME $\to$ NAV $\to$ HYE $\to$ AFB $\to$ SQI $\to$ MKD $\to$ AMH, are morphological inflection in Northern Sami, Navajo, Armenian, Gulf Arabic, Albanian, Macedonian and Amharic from the SIGMORPHON 2023 shared task~\citep{goldman2023sigmorphon} (pinned commit; \texttt{lemma, features}$\to$\texttt{form}). Each task uses 5{,}000 training examples in a chat template with thinking disabled; the model must output only the inflected form. Train and test lemma strings are disjoint (exact string match). We chose inflection because it is genuinely new knowledge for the base model in most of these languages, admits unambiguous exact-match scoring, and yields short answers whose first token dominates correctness.

\paragraph{Update mechanisms.} All arms use identical data order and budget: 125 optimizer steps $\times$ global batch 40 $=5{,}000$ presentations per task, max length 128, Adam ($\beta=(0.9,0.95)$, $\epsilon=10^{-8}$), peak learning rate $10^{-4}$ with cosine decay to zero, no warm-up, no weight decay, gradient clipping at 1. Optimizer and scheduler are reset at every task boundary.
\begin{itemize}[leftmargin=*,itemsep=2pt]
  \item \ngram{}: only the table rows addressed by the current task's training inputs are trainable, starting from their current (pretrained or previously updated) values. Each task has between 51{,}018 (AMH) and 149{,}809 (AFB) rows, i.e.\ 8--24M scalars. All seven tasks together touch 666{,}747 rows (0.21\% of the table), and each task shares between 1{,}277 and 4{,}911 rows with earlier tasks. Hashing, gate, read-out projections, router, all Transformer weights and the head are frozen. On the touched rows, Adam is exactly equivalent to full-table Adam, because untouched rows have identically zero gradient and optimizer state. Row values accumulate across tasks.
  \item \shared{} LoRA: rank 16, $\alpha=64$, no dropout, on six projection families (\texttt{in\_proj, out\_proj, linear\_qkv, linear\_proj, linear\_fc1, linear\_fc2}) covering linear attention, full attention, and \emph{every} routed and shared expert. The same factors continue training on each new task.
  \item \merged{} LoRA: identical configuration, but a freshly initialised adapter per task that is merged into the backbone weights at the end of the task.
\end{itemize}
The arms are matched in data exposure, not in parameter count: a full rank-16 adapter bank has 2.79B parameters (80.2M on any single token's path), while a task's N-gram support is $\sim$16M scalars. We return to this in \cref{sec:limitations}.

\paragraph{Probes and captures.} At $F_0$ and after every task we run forward-only passes on fixed probes:
\begin{itemize}[leftmargin=*,itemsep=2pt]
  \item \emph{Generic}: 128 MATH-500 problems~\citep{hendrycks2021math,lightman2024verify} with official solutions and 64 HumanEval problems~\citep{chen2021humaneval} with canonical solutions, in raw (template-free) format.
  \item \emph{Task}: 64 fixed held-out probes for each task learned so far, in the training template.
\end{itemize}
We record teacher-forced token-weighted NLL of the reference answer and capture, at the last prompt token (and at 25/50/75/100\% answer positions for math/code), the output of all 48 decoder layers (four streams concatenated), the post-merge representation and the LM-head input. For the final checkpoints we additionally capture the input and aggregate output of every MLP/MoE block. Drift uses pooled relative RMS (\cref{eq:relrms}); we also report mean cosine, pooled norm ratio, and linear CKA.

\paragraph{Generative evaluation.} Checkpoints are exported to Hugging Face format: the N-gram arm with its accumulated rows written into the table, \shared{} with its current factors merged once, and \merged{} with its per-task adapters merged in order. They are then decoded greedily (\texttt{do\_sample=false}) with vLLM (TP8, BF16, N-gram table offloaded to CPU). Task accuracy is exact match on the 64 task probes. Math accuracy (the same 128 problems) uses symbolic answer verification (math-verify 0.9.0). Code accuracy (the same 64 problems) is pass@1 under WASI-sandboxed execution; the sandbox passes all 164 canonical HumanEval solutions.

\paragraph{Second model.} An earlier, separately run chain on Llama-3-8B~\citep{dubey2024llama3} uses the same seven tasks and order with a different memory arm: a zero-initialised hashed bigram/trigram lookup table ($2\times262{,}144$ buckets of width 4{,}096) added at the input of the second decoder block. Llama-3-8B has no native memory: the table is new, and because it starts at zero the model at the start of training is exactly the base, whose existing knowledge does not depend on the table at all. This differs from Qwen3.8, whose table was pretrained together with the backbone and already carries part of the model's function, so that updating its rows edits an existing function. It also uses rank-8 LoRA on all seven projection types, 625 steps $\times$ batch 8 at learning rate $2\times10^{-4}$, seed 17, and representations after the final RMSNorm. We use it only to check whether the direction of the Qwen findings replicates (\cref{app:llama}).

\paragraph{Infrastructure.} Training uses Megatron-SWIFT on $4\times$H200 with TP2/PP2/EP2. Checkpoints are verified by hash before being used as the parent of the next task. All numbers are from a single training seed (data seed 42); probe sets are fixed across methods and rounds.

\section{Representation Drift Results}
\label{sec:results}

All three arms completed the seven-round chain with identical boundary evaluations. Unless stated otherwise, numbers are at the LM-head input, last prompt token; drift is pooled relative RMS (\cref{eq:relrms}).

\subsection{Big $K$ after seven rounds}
\label{sec:res-bigk}

\begin{table}[htbp]
\centering
\scriptsize
\setlength{\tabcolsep}{3.5pt}
\caption{Big-$K$ drift $F_0\to F_7$ (head input, last prompt token). Radial and directional energies are the two terms of \cref{eq:radial}, pooled and normalised by base energy; their sum is the squared relative RMS. $\Delta$NLL is the teacher-forced reference-answer loss change relative to $F_0$.}
\label{tab:bigk}
\begin{tabular}{llrrrrrrrrrr}
\toprule
Method & Domain & Rel.\ RMS & Mean cos & Median & P95 & Norm ratio & Radial & Direct. & Dir.\ share & CKA & $\Delta$NLL \\
\midrule
\ngram  & math & \textbf{26.37\%} & 0.9658 & 13.48\% & 58.37\% & 1.0059 & 0.0020 & 0.0676 & 97.17\% & 0.627 & $-0.0119$ \\
        & code & \textbf{15.13\%} & 0.9889 & 12.49\% & 23.58\% & 1.0019 & 0.0003 & 0.0226 & 98.66\% & 0.986 & $-0.0022$ \\
\shared & math & 97.33\% & 0.4502 & 97.75\% & 105.27\% & 0.8274 & 0.0358 & 0.9116 & 96.22\% & 0.366 & $-0.1001$ \\
        & code & 38.90\% & 0.9263 & 34.76\% & 54.89\% & 1.0003 & 0.0014 & 0.1499 & 99.09\% & 0.896 & $+0.0802$ \\
\merged & math & 64.36\% & 0.7894 & 52.64\% & 113.52\% & 0.9275 & 0.0143 & 0.3998 & 96.54\% & 0.127 & $-0.0865$ \\
        & code & 34.98\% & 0.9387 & 32.58\% & 51.61\% & 0.9639 & 0.0026 & 0.1198 & 97.86\% & 0.913 & $+0.0246$ \\
\bottomrule
\end{tabular}
\end{table}

\Cref{tab:bigk} is the central representation result. On both generic domains the ordering is
\[
\ngram\ <\ \merged\ <\ \shared ,
\]
and it holds for the pooled RMS, the median and the 95th percentile of per-query drift, and the mean cosine. Three further observations:
\begin{enumerate}[leftmargin=*,itemsep=2pt]
  \item \textbf{Drift is directional.} In every arm and domain more than 96\% of squared drift is the directional term $2\rho(1-c)$ of \cref{eq:radial}. \ngram{} keeps norms almost exactly ($\rho\approx1.00$). \shared{} on math is the only case with a substantial radial component: the pooled norm shrinks by 17\%.
  \item \textbf{Drift is not NLL.} \shared{} moves math representations the most (97\%) and also lowers math reference NLL the most ($-0.100$). Code NLL rises for both LoRA arms. Representation change at a generic probe can be neutral or even helpful for the loss, as \cref{eq:second-order} allows.
  \item \textbf{CKA disagrees with direct drift.} \merged{} has the lowest math CKA (0.13) although its direct drift is below \shared's. CKA measures sample-geometry similarity up to rotation and scale and is not a substitute for drift in the head's coordinates.
\end{enumerate}

\begin{figure}[htbp]
\centering
\includegraphics[width=\linewidth]{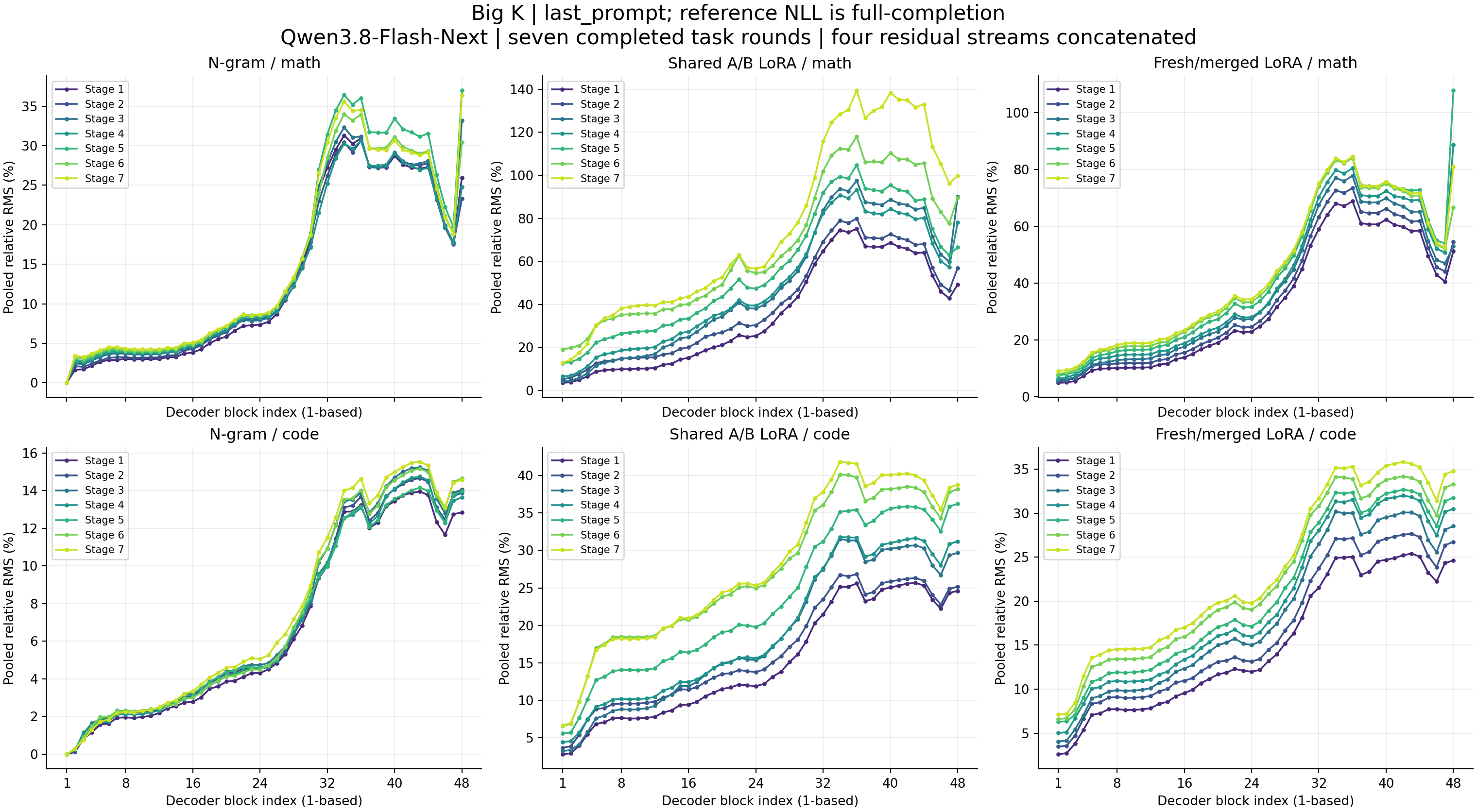}
\caption{Big-$K$ drift by layer for the seven rounds (one column per method, one row per generic domain; one curve per round). Layer curves are four residual streams concatenated (10{,}240 dimensions).}
\label{fig:bigk-layers}
\end{figure}

\subsection{Seven-round trajectory}
\label{sec:res-rounds}

\begin{table}[htbp]
\centering
\small
\caption{Big-$K$ trajectory at the head input: drift from $F_0$ after each round, and reference NLL (base: math 0.888, code 0.201).}
\label{tab:rounds}
\begin{tabular}{l r rrr rrr}
\toprule
 & & \multicolumn{3}{c}{Math} & \multicolumn{3}{c}{Code} \\
\cmidrule(lr){3-5}\cmidrule(lr){6-8}
Method & Round & Drift & Cos & NLL & Drift & Cos & NLL \\
\midrule
\ngram  & 1 & 25.21\% & 0.968 & 0.881 & 13.40\% & 0.991 & 0.200 \\
        & 3 & 24.81\% & 0.970 & 0.878 & 14.66\% & 0.990 & 0.199 \\
        & 5 & 26.14\% & 0.966 & 0.874 & 14.22\% & 0.990 & 0.199 \\
        & 7 & 26.37\% & 0.966 & 0.876 & 15.13\% & 0.989 & 0.199 \\
\shared & 1 & 45.51\% & 0.893 & 0.750 & 24.91\% & 0.969 & 0.210 \\
        & 3 & 74.71\% & 0.717 & 0.769 & 30.10\% & 0.955 & 0.217 \\
        & 5 & 66.64\% & 0.760 & 0.619 & 35.53\% & 0.937 & 0.205 \\
        & 7 & 97.33\% & 0.450 & 0.788 & 38.90\% & 0.926 & 0.281 \\
\merged & 1 & 45.01\% & 0.897 & 0.781 & 25.12\% & 0.969 & 0.209 \\
        & 3 & 52.41\% & 0.856 & 0.796 & 29.37\% & 0.958 & 0.215 \\
        & 5 & 76.22\% & 0.714 & 0.916 & 32.41\% & 0.948 & 0.224 \\
        & 7 & 64.36\% & 0.789 & 0.804 & 34.98\% & 0.939 & 0.224 \\
\bottomrule
\end{tabular}
\end{table}

\Cref{tab:rounds} (all rounds in \cref{app:rounds}) shows three different temporal profiles. \ngram{} jumps to about 25\% after the first task and then stays flat (21--26\% on math for all seven rounds). \shared{} grows almost monotonically on math (45\%$\to$97\%) and monotonically on code. \merged{} is non-monotone: math drift peaks at round 5 (76\%) and falls back to 64\% by round 7. After the first round, \shared{} and \merged{} are nearly indistinguishable (45.5\% vs 45.0\%), as they must be: the first task trains an identically configured adapter in both. The arms diverge only once \shared{} starts rewriting its factors.

\subsection{Accumulation and cancellation across rounds}
\label{sec:res-cancel}

\begin{table}[htbp]
\centering
\small
\caption{Cross-round accumulation of big-$K$ drift (\cref{eq:R}). $\sum_jE_j$ is the sum of per-round squared changes, normalised by base energy; $R_7<1$ means the cross-round inner products sum to a negative number.}
\label{tab:cancel}
\begin{tabular}{llrrr}
\toprule
Method & Domain & Final drift & $\sum_{j\le7}E_j$ & $R_7$ \\
\midrule
\ngram  & math & 26.37\% & 0.4781 & \textbf{0.145} \\
        & code & 15.13\% & 0.1207 & \textbf{0.190} \\
\shared & math & 97.33\% & 1.6958 & 0.559 \\
        & code & 38.90\% & 0.3865 & 0.392 \\
\merged & math & 64.36\% & 1.0220 & 0.405 \\
        & code & 34.98\% & 0.3019 & 0.405 \\
\bottomrule
\end{tabular}
\end{table}

Using \cref{eq:telescope,eq:R} we recover the total cross-round inner product from archived pooled norms (\cref{tab:cancel}, \cref{fig:temporal}). In all arms $R_7<1$: successive task updates partly undo one another on the generic probes. The decisive point concerns \ngram. Its per-round changes are not small---$\sum_jE_j=0.48$ on math, about half of \merged's---but they cancel much more strongly ($R_7=0.145$). Its small endpoint drift is therefore not ``nothing changed in each round''; it is ``each round changed the generic representation and the next round largely moved it back.'' For long-running learning this is the property that matters: the table's effect on the base behaves like a bounded excursion rather than a random walk. The same back-and-forth acts on learned-task inputs, which is where its cost appears (\cref{sec:res-smallk}).

\begin{figure}[htbp]
\centering
\includegraphics[width=\linewidth]{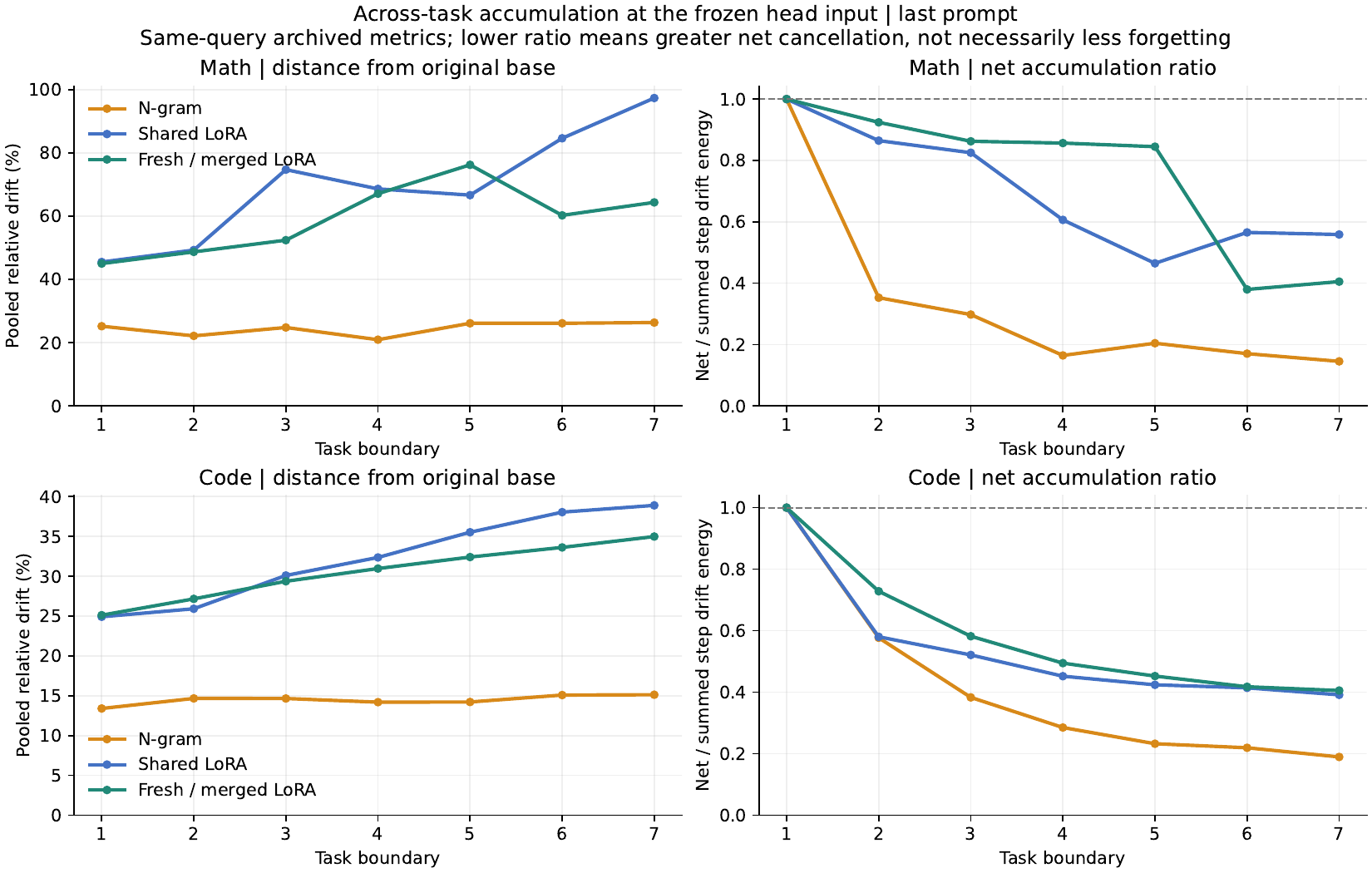}
\caption{Left: big-$K$ drift relative to $F_0$ over seven rounds. Right: cosine between round $j$'s change and the change accumulated before it, $(D_j-D_{j-1}-E_j)/(2\sqrt{D_{j-1}E_j})$.}
\label{fig:temporal}
\end{figure}

\subsection{Depth: residual streams and MLP/MoE blocks}
\label{sec:res-depth}

\begin{table}[htbp]
\centering
\small
\caption{Big-$K$ drift $F_0\to F_7$ at sampled layer outputs (four residual streams concatenated).}
\label{tab:layers}
\begin{tabular}{llrrrrrl}
\toprule
Method & Domain & L1 & L12 & L24 & L36 & L48 & Max layer \\
\midrule
\ngram  & math & 0.00\% & 4.25\% & 8.62\% & 34.52\% & 36.39\% & L48 (36.39\%) \\
        & code & 0.00\% & 2.50\% & 5.06\% & 14.63\% & 14.57\% & L43 (15.53\%) \\
\shared & math & 12.76\% & 39.50\% & 56.58\% & 139.45\% & 99.73\% & L36 (139.45\%) \\
        & code & 6.68\% & 18.45\% & 25.35\% & 41.55\% & 38.78\% & L34 (41.82\%) \\
\merged & math & 9.09\% & 19.02\% & 34.40\% & 84.60\% & 81.09\% & L36 (84.60\%) \\
        & code & 7.13\% & 14.76\% & 19.78\% & 35.31\% & 34.80\% & L42 (35.84\%) \\
\bottomrule
\end{tabular}
\end{table}

\Cref{tab:layers,fig:bigk-layers} show where in depth the arms separate. \ngram{} drift is exactly zero at layer 1, which precedes the memory read, and stays below 10\% through layer 25 even though the table is read at layer 2. It grows mainly in the last third of the network. Both LoRA arms change every layer directly, so drift is already 7--13\% at layer 1 and rises with depth. All arms peak between layers 34 and 48, and for both LoRA arms on math the drift falls after its peak (\shared{} 139\%$\to$100\% from L36 to L48). This is the behaviour described by \cref{eq:layer}: upstream drift is propagated, re-mixed and partly cancelled by later blocks, so drift norms need not increase monotonically with depth.

At the four final endpoints (base and the three $F_7$) we additionally captured the input and the aggregated output of every MLP/MoE block, before the hyper-connection injection, for 640 fixed queries (\cref{tab:mlp,fig:mlp}).

\begin{table}[htbp]
\centering
\small
\caption{Big-$K$ drift $F_0\to F_7$ at MLP/MoE blocks (\%). ``Median'' is the median over the 48 per-layer pooled drifts. Input and output are each normalised by their own base energy, so the two columns must not be subtracted.}
\label{tab:mlp}
\begin{tabular}{lll rrr}
\toprule
Domain & Hook & Method & L1 & L48 & Median \\
\midrule
math & input  & \ngram  & 0.05 & 22.21 & 10.59 \\
     &        & \shared & 14.19 & 88.73 & 72.91 \\
     &        & \merged & 8.34 & 58.05 & 44.41 \\
     & output & \ngram  & 0.09 & 74.50 & 19.72 \\
     &        & \shared & 17.93 & 123.51 & 112.13 \\
     &        & \merged & 11.48 & 151.05 & 69.87 \\
\midrule
code & input  & \ngram  & 0.09 & 14.50 & 6.15 \\
     &        & \shared & 11.43 & 37.55 & 29.76 \\
     &        & \merged & 7.59 & 34.61 & 23.88 \\
     & output & \ngram  & 0.12 & 12.19 & 11.31 \\
     &        & \shared & 7.86 & 33.91 & 43.52 \\
     &        & \merged & 8.44 & 30.45 & 36.14 \\
\bottomrule
\end{tabular}
\end{table}

The method ordering of \cref{tab:bigk} holds block by block, at both hooks. On math the median output drift exceeds the median input drift in every arm (\ngram{} 10.6\%$\to$19.7\%, \shared{} 72.9\%$\to$112.1\%, \merged{} 44.4\%$\to$69.9\%). For \ngram{} this is a pure propagation effect, since its MoE parameters are frozen and only the second term of \cref{eq:mlp-split} vanishes. At the last layer the \ngram{} MoE output drifts by 74.5\% on math although its input drifts by only 22.2\%. For the LoRA arms both terms of \cref{eq:mlp-split} are present.

\begin{figure}[htbp]
\centering
\includegraphics[width=0.9\linewidth]{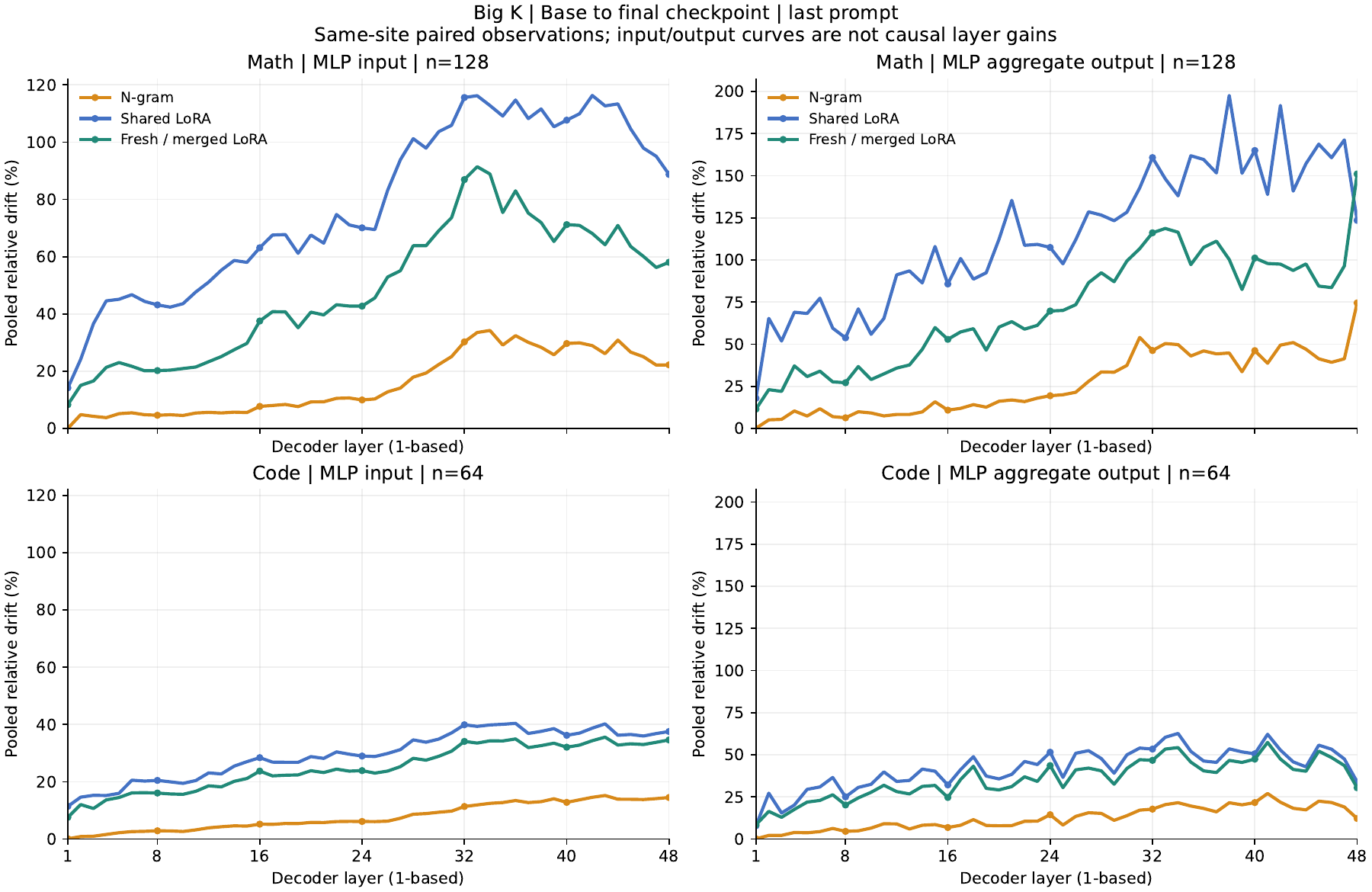}
\caption{Per-layer big-$K$ drift $F_0\to F_7$ at MLP/MoE block input and aggregated output.}
\label{fig:mlp}
\end{figure}

\subsection{Answer positions}
\label{sec:res-positions}

For math and code we also captured the positions at 25/50/75/100\% of the reference answer under teacher forcing. Head-input big-$K$ drift is position-dependent (\cref{tab:positions}): it is lower inside the answer than at the last prompt token for every arm (e.g.\ \shared{} math 97\% at the prompt vs 50--63\% inside the answer), and largest at the final answer position for both LoRA arms. The method ordering is unchanged at every position.

\begin{table}[htbp]
\centering
\small
\caption{Big-$K$ drift $F_0\to F_7$ at the head input by reference-answer position.}
\label{tab:positions}
\begin{tabular}{ll rrrr}
\toprule
Method & Domain & q25 & q50 & q75 & q100 \\
\midrule
\ngram  & math & 12.66\% & 13.66\% & 20.40\% & 15.00\% \\
        & code & 7.95\% & 10.65\% & 10.68\% & 17.95\% \\
\shared & math & 50.49\% & 53.08\% & 54.36\% & 62.74\% \\
        & code & 21.98\% & 26.21\% & 27.95\% & 53.03\% \\
\merged & math & 36.69\% & 40.50\% & 39.90\% & 46.41\% \\
        & code & 20.47\% & 23.21\% & 23.28\% & 43.42\% \\
\bottomrule
\end{tabular}
\end{table}

\subsection{Small $k_i$: from acquisition to the end of the chain}
\label{sec:res-smallk}

\begin{table}[htbp]
\centering
\scriptsize
\caption{Small-$k_i$ drift $F_i\to F_7$ on each learned task's 64 probes (head input, last prompt token), with the change in teacher-forced reference NLL over the same interval. AMH is learned last and has no subsequent training.}
\label{tab:smallk}
\begin{tabular}{ll rrrrrrr}
\toprule
Task & Method & Rel.\ RMS & Mean cos & Median & P95 & Norm ratio & Dir.\ share & $\Delta$NLL \\
\midrule
SME & \ngram  & 40.43\% & 0.9205 & 38.00\% & 56.79\% & 1.0099 & 98.18\% & $+0.236$ \\
    & \shared & 50.65\% & 0.8836 & 47.32\% & 67.03\% & 1.0703 & 96.83\% & $+0.356$ \\
    & \merged & \textbf{37.19\%} & 0.9313 & 36.03\% & 49.12\% & 0.9972 & 99.01\% & $\mathbf{+0.210}$ \\
NAV & \ngram  & 61.98\% & 0.8172 & 61.78\% & 77.22\% & 1.0465 & 99.07\% & $+0.753$ \\
    & \shared & 60.60\% & 0.8223 & 57.86\% & 78.65\% & 1.0262 & 99.11\% & $+1.018$ \\
    & \merged & \textbf{47.43\%} & 0.8881 & 46.34\% & 59.10\% & 1.0022 & 99.37\% & $\mathbf{+0.450}$ \\
HYE & \ngram  & 32.61\% & 0.9462 & 30.37\% & 44.01\% & 0.9577 & 97.54\% & $\mathbf{+0.082}$ \\
    & \shared & 28.37\% & 0.9599 & 25.48\% & 43.61\% & 0.9963 & 99.57\% & $+0.123$ \\
    & \merged & \textbf{25.84\%} & 0.9663 & 23.67\% & 34.99\% & 0.9585 & 96.80\% & $+0.100$ \\
AFB & \ngram  & 37.26\% & 0.9320 & 36.31\% & 51.61\% & 1.0155 & 99.29\% & $+0.353$ \\
    & \shared & 30.30\% & 0.9542 & 29.05\% & 39.09\% & 0.9957 & 99.14\% & $+0.163$ \\
    & \merged & \textbf{22.17\%} & 0.9753 & 21.66\% & 28.30\% & 0.9901 & 98.99\% & $\mathbf{+0.094}$ \\
SQI & \ngram  & 36.77\% & 0.9324 & 33.21\% & 57.49\% & 0.9982 & 99.61\% & $+0.215$ \\
    & \shared & 36.30\% & 0.9359 & 32.86\% & 52.75\% & 1.0236 & 98.98\% & $\mathbf{+0.080}$ \\
    & \merged & \textbf{30.31\%} & 0.9554 & 25.56\% & 52.39\% & 1.0115 & 99.52\% & $+0.088$ \\
MKD & \ngram  & 26.00\% & 0.9668 & 22.76\% & 39.94\% & 1.0092 & 99.51\% & $+0.192$ \\
    & \shared & 23.25\% & 0.9738 & 22.06\% & 30.11\% & 1.0242 & 98.46\% & $+0.057$ \\
    & \merged & \textbf{16.04\%} & 0.9874 & 15.45\% & 22.73\% & 1.0166 & 98.41\% & $\mathbf{+0.003}$ \\
\midrule
Mean & \ngram  & 39.18\% & & & & & & $+0.305$ \\
     & \shared & 38.25\% & & & & & & $+0.300$ \\
     & \merged & \textbf{29.83\%} & & & & & & $\mathbf{+0.157}$ \\
\bottomrule
\end{tabular}
\end{table}

\Cref{tab:smallk} measures what big $K$ cannot: how much the representations of \emph{learned} tasks move after they are learned. Three facts stand out.
\begin{enumerate}[leftmargin=*,itemsep=2pt]
  \item \textbf{The big-$K$ ordering does not transfer.} \ngram{} has the smallest big-$K$ drift but \emph{not} the smallest small-$k_i$ drift on any task; it is the largest of the three on five of six tasks (all but SME). Its mean small-$k_i$ drift (39.2\%) equals \shared's (38.3\%). Even with a frozen backbone, old-task representations move by 26--62\%.
  \item \textbf{\merged{} has the smallest small-$k_i$ drift on every task}, by 3--15 points.
  \item \textbf{Drift is again almost purely directional} (96.8--99.6\%), with norm ratios within 7\% of one.
\end{enumerate}
Older tasks drift more in every arm (SME and NAV have experienced six and five later tasks), but task age is not the only factor: HYE (third) drifts less than AFB (fourth) under \ngram.

\Cref{fig:layer-bigk-smallk} shows where in depth this happens. On generic inputs (a) the three arms are far apart at every layer, and \ngram{} stays below 10\% until layer 25. On learned-task inputs (b) the three curves nearly coincide, and \ngram{} drift appears immediately after the memory read: 19\% at layer 5 and 31\% at layer 9, against 4\% for generic inputs at the same layers. This is the single-source picture of \cref{prop:layers}: \ngram{} has one source term, the table injection, and it is small on generic inputs and large on old-task inputs, which read rows that later tasks rewrite. LoRA creates sources at every layer for both kinds of input. The two LoRA arms show the opposite asymmetry in deep layers. There they move generic inputs far more than learned-task inputs (\shared{} 139\% vs 54\% at layer 36).

\begin{figure}[htbp]
\centering
\includegraphics[width=\linewidth]{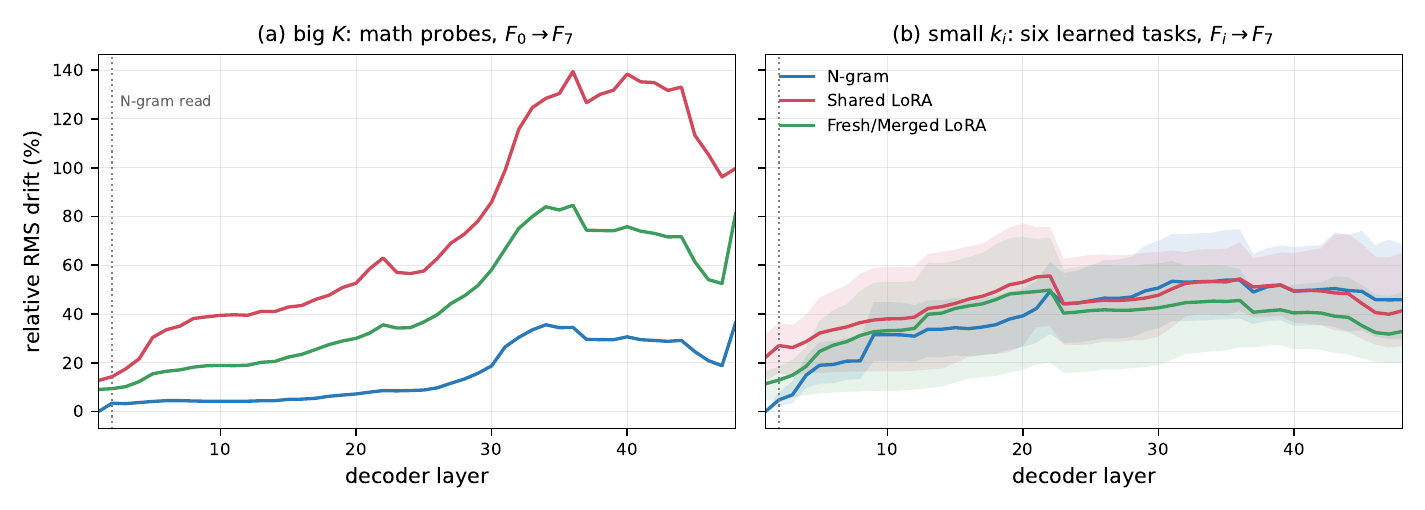}
\caption{Per-layer drift (four residual streams concatenated, last prompt token). (a) Big $K$ on math probes, $F_0\to F_7$. (b) Small $k_i$, mean over the six learned tasks with min--max band, $F_i\to F_7$. The dotted line marks the N-gram read at layer 2.}
\label{fig:layer-bigk-smallk}
\end{figure}

\begin{figure}[htbp]
\centering
\includegraphics[width=\linewidth]{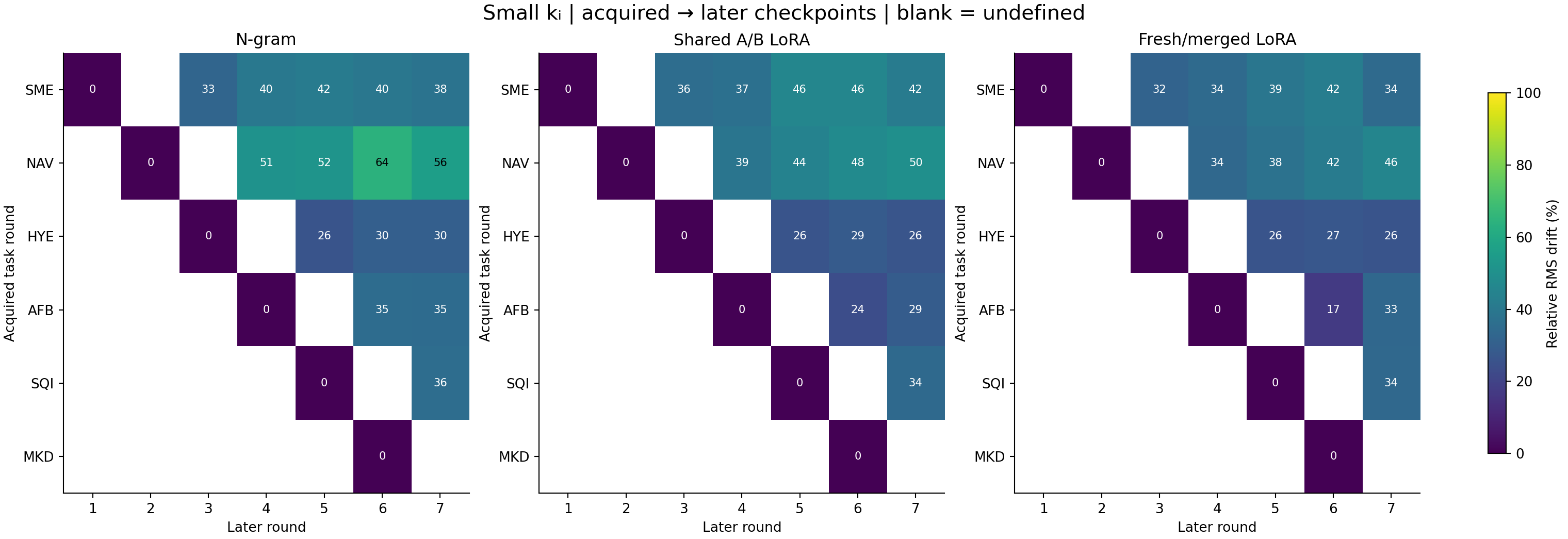}
\caption{Small-$k_i$ drift $F_i\to F_t$ for every learned task $i$ and later round $t$ (head input, last prompt token).}
\label{fig:smallk-triangle}
\end{figure}

\section{Behavioural Correlates and Boundary Checks}
\label{sec:behavior}

We now relate the drift measurements to behaviour: teacher-forced NLL, first-token margins, and greedy generation accuracy.

\subsection{Forgetting in NLL, and which drift tracks it}
\label{sec:res-dissociation}

\begin{table}[htbp]
\centering
\small
\caption{Final ($F_7$) reference NLL: generic probes relative to $F_0$, learned tasks relative to their acquisition checkpoint $F_i$ (mean over the six tasks with later training), next to the mean big-$K$ drift.}
\label{tab:forgetting}
\begin{tabular}{l rr rr r r}
\toprule
Method & Math NLL & $\Delta$ & Code NLL & $\Delta$ & Old-task mean $\Delta$NLL & Big-$K$ drift \\
\midrule
\ngram  & 0.8759 & $-0.0119$ & 0.1985 & $-0.0022$ & $+0.3052$ & \textbf{20.75\%} \\
\shared & 0.7877 & $-0.1001$ & 0.2809 & $+0.0802$ & $+0.2995$ & 68.12\% \\
\merged & 0.8036 & $-0.0865$ & 0.2244 & $+0.0246$ & $\mathbf{+0.1573}$ & 49.67\% \\
\bottomrule
\end{tabular}
\end{table}

\Cref{tab:forgetting} contains the paper's main dissociation. \ngram{} memory perturbs generic representations least and leaves generic NLL essentially unchanged. Yet on the tasks it has learned it forgets as much as \shared{} LoRA ($+0.305$ vs $+0.300$) and nearly twice as much as \merged{} LoRA ($+0.157$). The big-$K$ ordering $\ngram<\merged<\shared$ does not predict the forgetting ordering $\merged<\shared\approx\ngram$.

Small-$k_i$ drift does (\cref{fig:anchored}). Its means order the arms exactly as their mean $\Delta$NLL, and per task it tracks $\Delta$NLL across all 18 task--method pairs (Spearman $\rho=0.85$). This is what \cref{eq:dnll-bound} leads one to expect: forgetting of task $i$ is bounded by $D_i^t$, which is measured on task-$i$ inputs from $F_i$. Big-$K$ drift enters that bound only under the additional assumption that generic and task inputs move alike, and N-gram memory violates that assumption by construction (\cref{eq:ngram}).
The same dissociation appears on Llama-3-8B (\cref{fig:anchored}c,d; \cref{app:llama}). There the memory arm is not native: Llama-3-8B has no N-gram memory, so we added a lookup table initialised to zero, on which the base model's knowledge does not depend. Big-$K$ drift is again smallest for the memory arm (6.1\% vs 30.3\% merged, 50.5\% shared). It is also much smaller than for Qwen's native table (26\%/15\%), consistent with an empty side path perturbing the base less than editing rows that pretraining already uses. Its mean old-task $\Delta$NLL ($+0.69$) nevertheless exceeds \merged's ($+0.46$), and small-$k_i$ drift again tracks per-task $\Delta$NLL ($\rho=0.71$). The dissociation therefore holds whether the memory is native and pretrained or newly added and empty.
\begin{figure}[htbp]
\centering
\includegraphics[width=\linewidth]{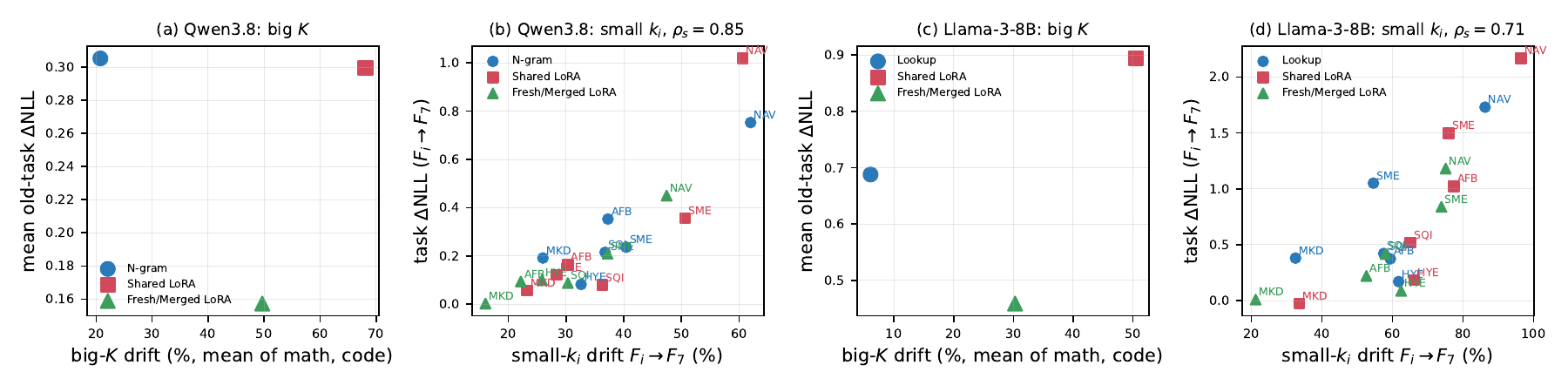}
\caption{Big-$K$ drift does not track forgetting; small-$k_i$ drift does. (a,c) Mean old-task $\Delta$NLL against big-$K$ drift: the memory arm has the smallest generic drift but not the least forgetting. (b,d) Per-task $\Delta$NLL against small-$k_i$ drift $F_i\to F_7$ for all 18 task--method pairs. (a,b) Qwen3.8 with native N-gram memory; (c,d) Llama-3-8B with an added zero-initialised lookup table.}
\label{fig:anchored}
\end{figure}

\subsection{First-token margins}
\label{sec:res-margin}

From the head inputs at $F_0$ and $F_7$ and the frozen head we computed, for every probe, the full-vocabulary margin of the first reference token, $m=z_y-\max_{c\neq y}z_c$, and its paired change $\Delta m=m_{\mathrm{final}}-m_{\mathrm{base}}$: 1{,}920 paired observations over 9 domains and 3 arms (\cref{tab:margin}, \cref{fig:margin}). On the seven tasks all 21 domain--method medians are positive, because this $F_0\to F_7$ comparison mixes learning with forgetting. The LoRA arms gain about twice the first-token margin of \ngram, and reference-token argmax rises from 69.6\% to 78.3\% (\ngram), 85.0\% (\shared) and 87.3\% (\merged). On math, where the first reference token is not a training target, both LoRA arms lower the margin while \ngram{} leaves it unchanged. Of the 10 math positions where the reference token strictly led at $F_0$, 4 (\ngram), 9 (\shared) and 10 (\merged) are overtaken at $F_7$. None of the 60 strictly leading code positions is lost in any arm.

\begin{table}[htbp]
\centering
\small
\caption{Median paired change of first-reference-token margin $F_0\to F_7$ (logits).}
\label{tab:margin}
\begin{tabular}{l rrr}
\toprule
Domain ($n$) & \ngram & \shared & \merged \\
\midrule
Math (128) & $+0.036$ & $-4.350$ & $-3.103$ \\
Code (64) & $+0.011$ & $-0.038$ & $+0.345$ \\
Seven tasks (448) & $+2.172$ & $+4.067$ & $+4.412$ \\
\bottomrule
\end{tabular}
\end{table}

\begin{figure}[htbp]
\centering
\includegraphics[width=0.8\linewidth]{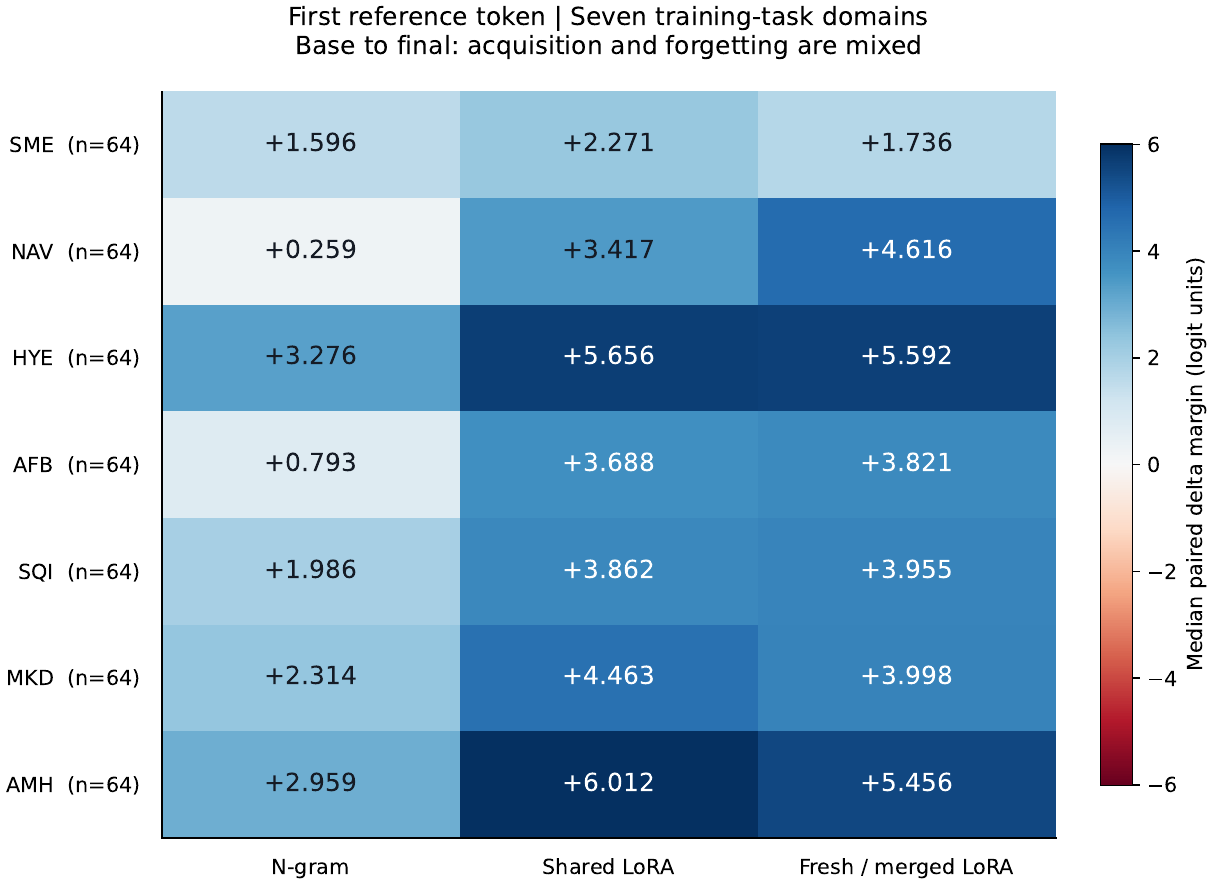}
\caption{Median first-reference-token margin change $F_0\to F_7$ by domain and method.}
\label{fig:margin}
\end{figure}

The pattern matches the big-$K$ picture: on code, which the table updates barely touch, N-gram memory leaves the base's decision margins where they were.

\subsection{Generative accuracy}
\label{sec:res-accuracy}

All endpoints were exported to Hugging Face format and decoded greedily with vLLM; every sample was scored ($\texttt{unresolved}=0$).

\begin{table}[htbp]
\centering
\small
\caption{Greedy exact match on the 64 probes of each task, from acquisition ($F_i$) to $F_7$ (percentage-point change).}
\label{tab:acc}
\begin{tabular}{l ccc}
\toprule
Task & \ngram & \shared & \merged \\
\midrule
SME & 13 $\to$ 13 \;(0.0) & 32 $\to$ 22 \;($-15.6$) & 31 $\to$ 27 \;($-6.3$) \\
NAV & \phantom{0}6 $\to$ \phantom{0}4 \;($-3.1$) & 23 $\to$ \phantom{0}8 \;($-23.4$) & 23 $\to$ 11 \;($-18.8$) \\
HYE & 61 $\to$ 55 \;($-9.4$) & 63 $\to$ 55 \;($-12.5$) & 64 $\to$ 57 \;($-10.9$) \\
AFB & 42 $\to$ 35 \;($-10.9$) & 52 $\to$ 49 \;($-4.7$) & 51 $\to$ 49 \;($-3.1$) \\
SQI & 50 $\to$ 45 \;($-7.8$) & 53 $\to$ 46 \;($-10.9$) & 51 $\to$ 48 \;($-4.7$) \\
MKD & 54 $\to$ 46 \;($-12.5$) & 59 $\to$ 57 \;($-3.1$) & 56 $\to$ 57 \;($+1.6$) \\
AMH & 57 & 58 & 60 \\
\midrule
Mean change (6 tasks) & $-7.3$ & $-11.7$ & $\mathbf{-7.0}$ \\
Cumulative accuracy at $F_7$ & 56.9\% & 65.8\% & \textbf{69.0\%} \\
\bottomrule
\end{tabular}
\end{table}

Task accuracy (\cref{tab:acc}) agrees with NLL and small-$k_i$ drift on the two LoRA arms: \shared{} loses more than \merged{} ($-11.7$ vs $-7.0$\,pp on average), with the earliest tasks hit hardest. For \ngram{} it exposes what drift and NLL do not. \ngram{} acquires the hard languages much more weakly (SME 13/64 vs 31--32; NAV 6/64 vs 23), so its small exact-match decline ($-7.3$\,pp) is bounded by a floor. In the language of \cref{prop:margin}, many \ngram{} probes never had a positive sequence margin to lose. Low drift combined with weak acquisition is not retention. Cumulative accuracy after round 7 is 56.9\% (\ngram), 65.8\% (\shared) and 69.0\% (\merged). Its round-3 jump in every arm comes from adding HYE, which all arms learn almost perfectly, so cumulative curves should not be read as forgetting curves.

\begin{table}[htbp]
\centering
\small
\caption{Generic capabilities at $F_7$: greedy accuracy on 128 MATH-500 problems and 64 HumanEval problems (pass@1, sandboxed WASI execution), with paired per-problem transitions relative to the base, and the number of math outputs whose final answer could not be parsed (counted as wrong).}
\label{tab:mathcode}
\begin{tabular}{l rrr rr}
\toprule
 & Math acc. & Lost / gained & Unparsed & Code acc. & Lost / gained \\
\midrule
Base & 71/128 (55.5\%) & -- & 56 & 51/64 (79.7\%) & -- \\
\ngram  & 68/128 (53.1\%) & 10 / 7 & 59 & 55/64 (85.9\%) & 0 / 4 \\
\shared & 32/128 (25.0\%) & 47 / 8 & 96 & 52/64 (81.3\%) & 7 / 8 \\
\merged & 94/128 (73.4\%) & 1 / 24 & 33 & 55/64 (85.9\%) & 6 / 10 \\
\bottomrule
\end{tabular}
\end{table}

Generic generation (\cref{tab:mathcode}) shows that changes are not uniformly degradations. Most of the math swings coincide with the number of unparsable final answers (base 56, \shared{} 96, \merged{} 33). They therefore reflect changes in answer format and termination behaviour at least as much as changes in mathematical competence. \shared{} lowers math reference NLL ($-0.100$) while losing 47 problems: teacher-forced NLL scores one fixed reference solution, whereas generation follows the model's own trajectory and an early format change propagates. This is exactly the gap between \cref{eq:ce-bound} and \cref{eq:argmax-cert}. \ngram{} changes generic behaviour least on math (10 lost / 7 gained), consistent with its small big-$K$ drift.

\section{Discussion: Write Online Knowledge into the Memory}
\label{sec:discussion}

\paragraph{The base is the binding constraint.}
An online learner receives an unbounded sequence of updates. Any channel whose perturbation of the base grows with the number of updates will eventually erode the capabilities the model was deployed for, however well it learns each update. The three channels differ sharply in this respect. Shared LoRA's generic drift grows almost monotonically (45\%$\to$97\% on math). Merged LoRA's also grows, non-monotonically, to 64\% with a peak of 76\%. N-gram memory reaches 25\% after the first task and stays there (\cref{tab:rounds}), because its per-round changes cancel ($R_7=0.15$, \cref{tab:cancel}). In generic behaviour it loses no code problem and changes math NLL by $-0.01$. Among the three, only N-gram memory has a base perturbation that behaves like a bounded excursion rather than an accumulating walk.

\paragraph{Why the memory protects the base.}
\Cref{eq:ngram} makes this structural rather than empirical. The table is read through addresses fixed by raw tokens, and everything upstream of the read is frozen. An input none of whose prefix rows were written is therefore exactly unchanged, at every layer and in every logit. General math and code text rarely addresses the rows a new task writes, so generic drift stays small and does not accumulate. In addition, every change lives in an explicit set of rows: 8--24M values per task, 0.21\% of the table for all seven tasks. These rows can be inspected, versioned and rolled back without touching the 125B-parameter backbone. Neither LoRA variant offers this: their updates act on every input through every adapted projection.

\paragraph{Where the cost appears, and why it belongs to the memory.}
N-gram memory pays in two places. It acquires difficult tasks more weakly (SME 13/64 vs 31--32 for LoRA), and its stored tasks interfere: old-task NLL rises as much as with shared LoRA, and small-$k_i$ drift is as large. \Cref{eq:ngram} also locates this cost. Old-task inputs share rows with later tasks: 1{,}277--4{,}911 rows per task by count, and in particular the template and chat-format rows that every example reads, including at the last prompt token. Later tasks rewrite those rows, and the layer profile shows it: on old-task inputs \ngram{} drift reaches 31\% within seven layers of the read, while generic inputs at the same depth move by 4\% (\cref{fig:layer-bigk-smallk}). Both weaknesses are thus properties of how the memory is addressed and written, not of the backbone. They can be attacked inside the memory: by keeping shared template rows fixed or task-specific, by allotting more rows or a larger learning rate to new content, or by per-row update rules. None of these requires editing the backbone again.

\paragraph{What merged LoRA shows.}
Merged LoRA retains learned tasks best and acquires them as well as shared LoRA. It also improves code generation, and its large math accuracy gain coincides with fewer unparsable answers. Its old factors are never edited, and its cumulative rank grows as $tr$, which avoids the overwrite terms of \cref{eq:shared-dW} and the capacity bound of \cref{eq:capacity}. These are good properties for an offline consolidation step, in which knowledge that has proved stable is periodically folded into the weights with full regression testing. As an always-on online channel it moves the base by 50\% on generic inputs after seven tasks, and that drift has no mechanism that bounds it.

\paragraph{A design for long-running online learning.}
Our measurements support a two-speed design. Online updates are written only to the native memory table, with the backbone and head frozen, so the base stays fixed and every change is local and reversible. Memory-side mechanisms then manage interference between stored items. Weight-level consolidation, if used at all, is an occasional offline step. Small-$k_i$ drift and acquisition margins, rather than generic drift alone, are the quantities to monitor, because forgetting follows them (\cref{fig:anchored}).

\paragraph{Limitations.}
\label{sec:limitations}
Results come from one task order and one seed on Qwen3.8, with Llama-3-8B as a second, differently configured model. The arms are matched in data and optimisation, not in parameter count. Margins are measured at the first reference token from the base, not from each acquisition checkpoint. Layer- and block-level results are observational.

\section{Conclusion}
\label{sec:conclusion}

We measured, layer by layer and round by round, how three ways of writing new knowledge change a 48-layer MoE language model over a seven-task continual chain. Updating only the native N-gram memory changes the base model's representations two to four times less than LoRA. That change does not grow with the number of tasks, because successive updates largely cancel, and it leaves generic behaviour essentially intact. Its cost, weaker acquisition and interference among stored tasks, shows up on the learned tasks' own inputs. There forgetting follows small-$k_i$ drift, and exact margin relations explain why. For a model that must keep learning online, these results point to a clear division of labour: write new knowledge only into the memory, keep the backbone fixed, and solve interference inside the memory rather than by editing the backbone.

\section*{Acknowledgments}
The author used AI assistants for drafting text, writing code for analysis and figures, and checking the literature. All experiments, results and claims were verified by the author, who takes full responsibility for the content.

\bibliographystyle{plainnat}
\bibliography{refs}

\appendix
\section{Llama-3-8B Replication}
\label{app:llama}

The Llama-3-8B chain (\cref{sec:setup}) differs from the Qwen chain in one essential respect: its memory arm is a lookup table added to a model that has none, initialised to zero, so the base model's knowledge does not depend on it and big-$K$ drift measures only the perturbation caused by a new side path. All captures used SDPA attention. Representations are taken after the final RMSNorm (the LM-head input), at the last prompt token, for 64 math and 32 code probes and 64 probes per task. All NLL changes are computed against the same-backend base or acquisition checkpoint. Absolute values are not comparable with Qwen; only within-model orderings are.

\begin{table}[htbp]
\centering
\small
\caption{Llama-3-8B big-$K$ drift $F_0\to F_7$ and reference NLL change.}
\label{tab:llama-bigk}
\begin{tabular}{l rrr rrr}
\toprule
 & \multicolumn{3}{c}{Math} & \multicolumn{3}{c}{Code} \\
\cmidrule(lr){2-4}\cmidrule(lr){5-7}
Method & Drift & Mean cos & $\Delta$NLL & Drift & Mean cos & $\Delta$NLL \\
\midrule
Lookup  & \textbf{4.86\%} & 0.9988 & $+0.0024$ & \textbf{7.29\%} & 0.9973 & $-0.0002$ \\
\shared & 69.10\% & 0.7661 & $+0.0712$ & 31.87\% & 0.9496 & $+0.0590$ \\
\merged$^\dagger$ & 39.04\% & 0.9245 & $+0.0266$ & 21.51\% & 0.9771 & $+0.0217$ \\
\bottomrule
\end{tabular}
\end{table}

\begin{table}[htbp]
\centering
\small
\caption{Llama-3-8B small-$k_i$ drift $F_i\to F_7$ and task reference $\Delta$NLL.}
\label{tab:llama-smallk}
\begin{tabular}{l rrr rrr}
\toprule
 & \multicolumn{3}{c}{Small-$k_i$ drift} & \multicolumn{3}{c}{$\Delta$NLL} \\
\cmidrule(lr){2-4}\cmidrule(lr){5-7}
Task & Lookup & \shared & \merged$^\dagger$ & Lookup & \shared & \merged$^\dagger$ \\
\midrule
SME & 54.58\% & 75.91\% & 73.86\% & $+1.051$ & $+1.496$ & $+0.838$ \\
NAV & 86.25\% & 96.38\% & 75.07\% & $+1.730$ & $+2.166$ & $+1.181$ \\
HYE & 61.74\% & 66.16\% & 62.45\% & $+0.170$ & $+0.186$ & $+0.087$ \\
AFB & 59.38\% & 77.42\% & 52.60\% & $+0.374$ & $+1.023$ & $+0.221$ \\
SQI & 57.52\% & 65.08\% & 57.99\% & $+0.422$ & $+0.518$ & $+0.417$ \\
MKD & 32.50\% & 33.50\% & 21.19\% & $+0.381$ & $-0.026$ & $+0.009$ \\
\midrule
Mean & 58.66\% & 69.08\% & 57.19\% & $+0.688$ & $+0.894$ & $+0.459$ \\
\bottomrule
\multicolumn{7}{l}{\footnotesize $^\dagger$Trained on the full training split; Lookup and \shared{} used a fixed 1{,}024-example subset per task.}
\end{tabular}
\end{table}

The direction of the Qwen findings replicates. The memory arm has by far the smallest big-$K$ drift but not the least forgetting. \merged{} forgets least in NLL and has the smallest mean small-$k_i$ drift, with Lookup essentially tied in drift. Small-$k_i$ drift tracks per-task $\Delta$NLL ($\rho=0.71$). Final seven-task accuracy is 34.2\% (Lookup), 33.9\% (\shared) and 49.8\% (\merged). Lookup's mean accuracy at acquisition (54.7\%) is again lower than \shared's (60.3\%), the same acquisition gap as on Qwen.

\section{Full Seven-Round Big-$K$ Trajectory}
\label{app:rounds}

\begin{table}[htbp]
\centering
\small
\caption{Qwen3.8 big-$K$ drift from $F_0$, mean cosine and reference NLL at every task boundary (head input, last prompt token). Base NLL: math 0.8878, code 0.2007.}
\begin{tabular}{l r rrr rrr}
\toprule
 & & \multicolumn{3}{c}{Math} & \multicolumn{3}{c}{Code} \\
\cmidrule(lr){3-5}\cmidrule(lr){6-8}
Method & Round & Drift & Cos & NLL & Drift & Cos & NLL \\
\midrule
\ngram & 1 & 25.21\% & 0.9682 & 0.8811 & 13.40\% & 0.9914 & 0.1995 \\
 & 2 & 22.16\% & 0.9750 & 0.8767 & 14.67\% & 0.9896 & 0.1990 \\
 & 3 & 24.81\% & 0.9695 & 0.8780 & 14.66\% & 0.9897 & 0.1987 \\
 & 4 & 20.95\% & 0.9778 & 0.8774 & 14.19\% & 0.9903 & 0.2003 \\
 & 5 & 26.14\% & 0.9664 & 0.8741 & 14.22\% & 0.9901 & 0.1994 \\
 & 6 & 26.15\% & 0.9657 & 0.8757 & 15.09\% & 0.9891 & 0.1988 \\
 & 7 & 26.37\% & 0.9658 & 0.8759 & 15.13\% & 0.9889 & 0.1985 \\
\midrule
\shared & 1 & 45.51\% & 0.8930 & 0.7504 & 24.91\% & 0.9692 & 0.2100 \\
 & 2 & 49.28\% & 0.8735 & 0.7410 & 25.92\% & 0.9671 & 0.2139 \\
 & 3 & 74.71\% & 0.7165 & 0.7685 & 30.10\% & 0.9553 & 0.2170 \\
 & 4 & 68.62\% & 0.7571 & 0.7472 & 32.36\% & 0.9487 & 0.2026 \\
 & 5 & 66.64\% & 0.7602 & 0.6192 & 35.53\% & 0.9372 & 0.2052 \\
 & 6 & 84.61\% & 0.6045 & 0.7725 & 38.05\% & 0.9285 & 0.2346 \\
 & 7 & 97.33\% & 0.4502 & 0.7877 & 38.90\% & 0.9263 & 0.2809 \\
\midrule
\merged & 1 & 45.01\% & 0.8973 & 0.7811 & 25.12\% & 0.9689 & 0.2088 \\
 & 2 & 48.70\% & 0.8775 & 0.7426 & 27.16\% & 0.9636 & 0.2057 \\
 & 3 & 52.41\% & 0.8559 & 0.7956 & 29.37\% & 0.9575 & 0.2151 \\
 & 4 & 67.10\% & 0.7752 & 0.8744 & 30.97\% & 0.9521 & 0.2166 \\
 & 5 & 76.22\% & 0.7144 & 0.9159 & 32.41\% & 0.9477 & 0.2237 \\
 & 6 & 60.25\% & 0.8137 & 0.8525 & 33.62\% & 0.9435 & 0.2237 \\
 & 7 & 64.36\% & 0.7894 & 0.8036 & 34.98\% & 0.9387 & 0.2244 \\
\bottomrule
\end{tabular}
\end{table}

\section{Reproducibility}
\label{app:repro}
Checkpoints are verified by SHA-256 before they are used as the parent of the next task. Probe sets, prompts and tokenisation protocols are fingerprinted, and repeated capture of the same input gives zero drift.

\end{document}